%% file: Schmerling.ea.LinearFMT.ARXIV.tex
\documentclass[letterpaper, 10 pt, conference]{ieeeconf}
\IEEEoverridecommandlockouts                              %

\makeatletter
\let\NAT@parse\undefined
\makeatother
\usepackage[numbers]{natbib}

\usepackage{amsmath}

\usepackage{amssymb}
\usepackage{amsthm}
\usepackage{graphicx}
\usepackage{times}
\usepackage{enumerate}
\usepackage{subfigure}

\graphicspath{{./fig/}}
\input{preamble}

\usepackage{amsmath}
\makeatletter

\usepackage{algorithmic}
\usepackage{algorithm}

\newcommand{\ssmargin}[2]{{#1}}

\usepackage{url}

\renewcommand{\baselinestretch}{0.94}
\title{Optimal Sampling-Based Motion Planning under Differential Constraints: the Drift Case with Linear Affine Dynamics \vspace{-3mm}}

\author{Edward Schmerling
    \thanks{Edward Schmerling is with the Institute for Computational \& Mathematical \ Engineering, Stanford University, Stanford, CA 94305, \texttt{schmrlng@stanford.edu}.}
    \and Lucas Janson
    \thanks{Lucas Janson is with the Department  of Statistics, Stanford University, Stanford, CA 94305, \texttt{ljanson@stanford.edu}.}
    \thanks{Marco Pavone is with the Department\ of Aeronautics and Astronautics, Stanford University, Stanford, CA 94305, \texttt{pavone@stanford.edu}.}
    \and Marco Pavone%
    \thanks{This work was supported by an Early Career Faculty grant from NASA's Space Technology Research Grants Program (Grant NNX12AQ43G).}
}

\begin{document}

\maketitle

\begin{abstract}
In this paper we provide a thorough, rigorous theoretical framework to assess optimality guarantees of sampling-based algorithms for drift control systems: systems that, loosely speaking, can not stop instantaneously due to momentum. We exploit this framework to design and analyze a sampling-based algorithm (the Differential Fast Marching Tree algorithm) that is asymptotically optimal, that is, it is guaranteed to converge, as the number of samples increases, to an optimal solution. In addition, our approach allows us to provide concrete bounds on the rate of this convergence. \ssmargin{The focus of this paper is on mixed time/control energy cost functions and on linear affine dynamical systems, which encompass a range of models of interest to applications (e.g., double-integrators) and represent a necessary step to design, via successive linearization, sampling-based and provably-correct algorithms for non-linear drift control systems}{too long}. Our analysis relies on an original perturbation analysis for two-point boundary value problems, which could be of independent interest.

\end{abstract}

\vspace{-0.1truecm}
\section{Introduction}\label{sec:intro}
A key problem in robotics is how to compute an obstacle-free and dynamically-feasible trajectory that a robot can execute \cite{SML:06}. The problem, in  the simplest setting where the robot does not have kinematic/dynamical (in short, differential) constraints on its motion and the problem becomes one of finding an obstacle-free ``geometric" path, is reasonably well-understood and sound algorithms exist for most practical scenarios. However, robotic systems \emph{do} have differential constraints (e.g., momentum), which most often cannot be neglected. Despite the long history of robotic motion planning, the inclusion of differential constraints in the planning process is currently considered an open challenge \cite{SML:11b}, in particular with respect to guarantees on the quality of the obtained solution and class of dynamical systems that can be addressed. Arguably, the most common approach in this regard is a decoupling approach, whereby the problem is decomposed in steps of computing a collision-free path (neglecting the differential constraints), smoothing the path to satisfy the motion constraints, and finally reparameterizing the trajectory so that the robot can execute it \cite{SML:11b}. This approach, while oftentimes fairly computationally efficient, presents a number of disadvantages, including computation of trajectories whose cost (e.g., length or control effort) is far from the theoretical optimum or even failure in finding any solution trajectory due to the decoupling scheme itself  \cite{SML:11b}. For these reasons, it has been advocated that there is a need for planning algorithms that solve the differentially-constrained motion planning problem (henceforth referred to as the DMP problem) \emph{in one shot}, i.e., without decoupling.

Broadly speaking, the DMP problem can be divided into two categories: (i) DMP for driftless systems, and (ii) DMP for drift systems. Intuitively, systems with drift constraints are systems where from some states it is impossible to stop instantaneously (this is typically due to momentum). More rigorously, a system $\dot{\x} = f(\x, \u)$ is a drift system if for some state $\x$ there does not exist any admissible control $\u$ such that $f(\x, \u) = 0$  \cite{SML:06}. For example the basic, yet representative, double integrator system $\ddot{\x} = \u$ (modeling the motion of a point mass under controlled acceleration) is a drift system. From a planning perspective, DMP for drift systems is more challenging than its driftless counterpart, due, for example, to the inherent lack of symmetries in the dynamics and the presence of regions of inevitable collision (that is, sets of states from which obstacle collision will eventually occur, regardless of applied controls) \cite{SML:06}.

To date, the state of the art for one-shot solutions to the DMP problem (both for driftless and drift systems) is represented by sampling-based techniques, whereby an explicit construction of the configuration space is avoided and the configuration space is probabilistically ``probed" with a sampling scheme \cite{SML:06}. Arguably, the most successful algorithm for DMP to date is the rapidly-exploring random tree algorithm (RRT) \cite{SML-JJK:01}, which incrementally builds a tree of trajectories by randomly sampling points in the configuration space. However, the RRT algorithm lacks optimality guarantees, in the sense that one can prove that the cost of the solution returned by RRT converges to a suboptimal cost as the number of sampled points goes to infinity, almost surely \cite{SK-EF:11}. An asymptotically-optimal version of RRT for the geometric (i.e., without differential constraints) case has been recently presented in \cite{SK-EF:11}. This version, named \RRTs, essentially adds a rewiring stage to the RRT algorithm to counteract its greediness in exploring the configuration space. Prompted by this result, a number of works have proposed extensions of \RRTs to the DMP problem  \cite{SK-EF:10,AP-RP-GK-LK-TLP:12,GG-AP-RP-GK:13,SK-EF:13,DW-JvdB:13}, with the goal of retaining the asymptotic optimality property of \RRTs. Care must be taken in arguing optimality for drift systems in particular, as the control asymmetry requires a consideration of both forward-reachable and backward-reachable trajectory approximations.
Even in the driftless case, the matter of assessing optimality is quite subtle, and hinges upon a careful characterization of  a system's locally reachable sets in order to ensure that a planning algorithm examines ``enough volume'' in its operation, and thus enough sample points, to ensure asymptotic optimality \cite{ES-LJ-MP:15a}.
Another approach to asymptotically optimal DMP planning is given by STABLE SPARSE RRT which achieves optimality through random control propagation instead of connecting sampled points using a steering subroutine \cite{YL-ZL-KB:14}. This paper, like the \RRTs variations, is based on a steering function, although it may be considered less general, as it is our view that leveraging as much knowledge as possible of the differential constraints while planning is necessary for the goal of planning in real-time. In our related work \cite{ES-LJ-MP:15a} we provide a theoretical framework to study optimality guarantees of sampling-based algorithms for the DMP problem by focusing 
  on \emph{driftless} control-affine dynamical systems of the form
$\dot \x(t) = \sum_{i=1}^m g_i(\x(t)) \u_i(t)$. While this model is representative for a large class of robotic systems (e.g., mobile robots with wheels that roll without slipping and multi-fingered robotic hands), it is of limited applicability  in problems where momentum (i.e., drift) is a key feature of the problem setup (e.g., for a spacecraft or a helicopter).  

\emph{Statement of Contributions}: The objective of this paper is to provide a theoretical framework to study optimality guarantees of sampling-based algorithms for the DMP problem with \emph{drift}. Specifically, as in \cite{DW-JvdB:13}, we focus on linear affine systems of the form
\begin{equation*}
\dot \x[t] = A \x[t] + B \u[t] + \c, \quad \x[t] \in \M,\, \u[t] \in \U,
\end{equation*}
where $ \M$ and $\U$ are the configuration and control spaces, respectively, and it is of interest to find an obstacle-free trajectory $\pi$ that minimizes the  mixed time/energy criterion
\begin{equation*}
c[\pi] = \int_{0}^T \left(1 + \u[t]^T R \u[t]\right)\dif t,
\end{equation*}
where $R$ is a positive definite matrix that weights control energy expenditure versus traversal time. Henceforth, we will refer to a DMP problem involving linear affine dynamics  and a mixed time/energy cost criterion  as Linear Quadratic DMP (LQDMP). The LQDMP problem is relevant to applications for two main reasons: (i) it models the ``essential" features of a number of robotic systems (e.g., spacecraft in deep space, helicopters, or even ground vehicles), and (ii) its theoretical study forms the backbone for sampling-based approaches that rely on linearization of more complex underlying dynamics. From a theoretical and algorithmic standpoint, the LQDMP problem  presents two challenging features: (i) dynamics are not symmetric \cite{SML:06}, which makes forward and backward reachable sets different and requires a more sophisticated analysis of sampling volumes to prove asymptotic optimality, and (ii) not all directions of motion are equivalent, in the sense that some motions incur dramatically higher cost than others due to the algebraic structure of the constraints. Indeed, these are the very same challenges that make the  DMP problem with drift difficult in the first place, and they make approximation arguments (e.g., those needed to prove asymptotic optimality) more involved. Fortunately, for LQDMP an explicit characterization for the optimal trajectory connecting two sampled points in the absence of obstacles is available, which provides a foothold to begin the analysis. Specifically, the contribution of this paper is threefold. First, we show that \emph{any} trajectory in an LQDMP problem may be ``traced" arbitrarily well, with high probability, by connecting randomly distributed points from a sufficiently large sample set covering the configuration space. We will refer to this property as \emph{probabilistic exhaustivity}, as opposed to probabilistic completeness \cite{SML:06}, where the requirement is that \emph{at least} one trajectory is traced with a sufficiently large sample set. Second, we introduce a sampling-based algorithm for solving the LQDMP problem, namely the Differential Fast Marching Tree algorithm (\DFMT), whose design is enabled by our analysis of the notion of probabilistic exhaustivity. In particular, we are able to give a precise characterization of neighborhood radius, an important parameter for many asymptotically optimal motion planners, in contrast with previous work on LQDMP \cite{DW-JvdB:13}. Third, by leveraging probabilistic exhaustivity, we show that \DFMT is asymptotically optimal. This analysis framework builds upon \cite{ES-LJ-MP:15a}, and elements of our approach are inspired by \cite{DW-JvdB:13}. We note that in \cite{DW-JvdB:13}, the authors present an excellent extension of \RRTs that successfully solves the LQDMP problem in simulations, even when extended to linearized systems. The asymptotic optimality claim, however, relies only on a near-neighbor set size argument: we aim to put the analysis of the LQDMP problem on more rigorous theoretical footing.

\emph{Organization}: This paper is structured as follows.  In Section \ref{sec:probformulation} we formally define the DMP problem we wish to solve. In Section 
\ref{sec:optcontrol} we review known results about the problem of optimally connecting fixed initial and terminal states under linear affine dynamics with a quadratic cost function. Furthermore, we provide a simple, yet novel (to the best of our knowledge) asymptotic characterization of the spectrum of the weighted controllability Gramian, which is instrumental to our analysis. In Section \ref{sec:exhaust} we prove the aforementioned probabilistic exhaustivity property for drift  systems with linear affine dynamics. In Section \ref{sec:algo} we present the  \DFMT algorithm, and in Section \ref{sec:AO} we discuss its asymptotic optimality (together with a convergence rate characterization). Section \ref{sec:sims} contains proof-of-concept simulations. Finally, in Section \ref{sec:disc} we discuss several features of our analysis, we draw some conclusions, and we discuss directions for future work.

\section{Problem Formulation}\label{sec:probformulation}
Let $\M \subseteq \R^n$ and $\U \subseteq \R^m$ be the configuration space and control space, respectively, of a robotic system. Within this space let us assume the dynamics of the robot are given by the linear affine system
\begin{equation}\label{eqn:linsys}
\dot \x[t] = A \x[t] + B \u[t] + \c, \quad \x[t] \in \M,\, \u[t] \in \U,
\end{equation}
where $A \in \R^{n\times n}$, $B \in \R^{n\times m}$, and $\c \in \R^n$ are constants.

A tuple $\pi=(\x[],\u[],T)$ defines a \emph{dynamically feasible} trajectory, alternatively path, if the state evolution $\x:[0,T]\rightarrow \M$ and control input $\u:[0,T]\rightarrow\U$ satisfy equation~\eqref{eqn:linsys} for all $t \in [0,T]$.
We define the cost of a trajectory $\pi$ by the function
\begin{equation}\label{eqn:cost}
c[\pi] = \int_{0}^T \left(1 + \u[t]^T R \u[t]\right)\dif t
\end{equation}
where $R \in \R^{m\times m}$ is symmetric positive definite, constant, and given. We may rewrite this cost function as $c[\pi]=T + c_u[\pi]\nonumber$, where $c_u[\pi] = \int_{0}^T \u[t]^T R \u[t]\dif t$, with the interpretation that this cost function penalizes both trajectory duration $T$ and control effort $c_u$. The matrix $R$ determines the relative costs of the control inputs, as well as their costs relative to the duration of the trajectory. We denote this linear affine dynamical system with cost by $\Sigma = (A, B, \c, R)$.

Let $\Mobs \subset \M$ be the obstacle region within the configuration space and consider the closed obstacle-free space $\Mfree = \cl[\M \setminus \Mobs]$. The starting configuration $\xinit$ is an element of $\Mfree$, and the goal region $\Mgoal$ is an open subset of $\Mfree$.  The  trajectory planning problem is denoted by the tuple $(\Sigma, \Mfree, \xinit, \Mgoal)$. A dynamically feasible trajectory $\pi=(\x,\u,T)$ is \emph{collision-free} if $\x[t] \in \Mfree$ for all $t\in [0,T]$. A trajectory $\pi$ is said to be \emph{feasible} for the trajectory planning problem $(\Sigma, \Mfree, \xinit, \Mgoal)$ if it is dynamically feasible, collision-free, $\x[0] = \xinit$, and $\x[T] \in \Mgoal$.

Let $\Pi$ be the set of all feasible paths. The objective is to find the feasible path with minimum associated cost. We define the optimal trajectory planning problem as follows:

\begin{quote}{\bf LQDMP problem}: 
Given a  trajectory planning problem $(\Sigma, \Mfree, \xinit, \Mgoal)$ with cost function $c:~\Pi \to \R_{\geq 0}$ given by equation~\eqref{eqn:cost}, find a feasible path $\pi^{*}$ such that $c[\pi^{*}]= \inf\{c[\pi] \mid \pi \text{ is feasible}\}$. If no such path exists, report failure.
\end{quote}
Our analysis will rely on two key sets of assumptions, relating, respectively, to the system $\Sigma$ and the problem-specific parameters $\Mfree, \xinit, \Mgoal$.

\emph{Assumptions on system}: We assume that the system $\Sigma$ is controllable, (i.e., the pair $(A,B)$ is controllable) \cite{CTC:95} so that even disregarding obstacles there exist dynamically feasible trajectories between states.\footnote{This system controllability assumption is why we do not fold the constant drift term $\c$ into the state $\x$.} Also, we assume that the control space is unconstrained, i.e. $\U = \R^m$, and that the cost weight matrix $R$ is symmetric positive definite, so that every control direction has positive cost. These assumptions will be collectively referred to as $A_\Sigma$.

\emph{Assumptions on problem parameters}: We require that the configuration space is a compact subset of $\R^n$ so that we may sample from it. Furthermore, we require that the goal region $\Mgoal$ has \emph{regular boundary}, that is there exists $ \xi > 0$ such that for almost all $\y \in \partial \Mgoal$, there exists $\z \in \Mgoal$ with $B[\z, \xi] \subseteq \Mgoal$ and $\y \in \partial B[\z, \xi]$, where $B$ denotes the Euclidean 2-norm ball. This requirement that the boundary of the goal region has bounded curvature almost everywhere ensures that a sampling procedure may expect to select points in the goal region near any point on the region's boundary.
We make requirements on the \emph{clearance} of the optimal trajectory, i.e.,  its ``distance" from $\Mobs$ \cite{ES-LJ-MP:15a}. For a given $\delta>0$, the $\delta$-interior of $\Mfree$ is the set of all states that are at least a Euclidean distance $\delta$ away from any point in $\Mobs$. A collision-free path $\pi$ is said to have strong $\delta$-clearance if its state trajectory $\x$ lies entirely inside the $\delta$-interior of $\Mfree$.
A collision-free path $\pi$ is said to have weak $\delta$-clearance if there exists a path $\pi'$ that has strong $\delta$-clearance and there exists a homotopy $\psi$, with $\psi[0]= \pi$ and $\psi[1] = \pi'$ that satisfies the following three properties: (a) $\psi[\alpha]$ is a dynamically feasible
trajectory for all $ \alpha \in (0, 1]$, (b) $\lim_{\alpha\rightarrow0} c[\psi[\alpha]] = c[\pi]$, and (c) for all $ \alpha \in (0, 1]$ there exists $\delta_{\alpha}>0$ such that $\psi[\alpha]$ has strong $\delta_{\alpha}$-clearance. Properties (a) and (b) are required since pathological obstacle sets may be constructed that squeeze all optimum-approximating homotopies into undesirable motion. In practice, however, as long as $\Mfree$ does not contain any passages of infinitesimal width, the fact that $\Sigma$ is controllable will allow every trajectory to be weak $\delta$-clear. We claim that these assumptions about the problem parameters are mild, and can be regarded as ``minimum'' regularity assumptions.

All trajectories discussed in this paper are dynamically feasible unless otherwise noted.
The symbol $\|\cdot\|$ denotes the 2-norm, induced or otherwise.
The asymptotic notations $O, \Omega, \Theta, o$ mean bounded above, bounded below, bounded both above and below, and asymptotically dominated, respectively.

\section{Optimal Control in the Absence of Obstacles}\label{sec:optcontrol}

The goal of this section is twofold: to review results about  two-point boundary value problems for linear affine systems, and to present a simple, yet novel asymptotic characterization of the spectrum of the controllability Gramian. Both results will be instrumental to our analysis of LQDMP.
\vspace{-.1cm}
\subsection{Two Point Boundary Value Problem}
The material in this section is standard, we provide it to make the paper self-contained. Our presentation follows the treatment in  \cite{FL-VS:95,DW-JvdB:13}. Specifically, this section is concerned with local steering between states in the absence of environment boundaries and obstacles. Given a start state $\x_0 \in \M$ and an end state $x_1 \in \M$, the \emph{two point boundary value problem} (2BVP) is to find a trajectory $\pi = (\x, \u, \T^*)$ between $\x[0] = \x_0$ and $\x[\T^*] = \x_1$ that satisfies the system $\Sigma$ and minimizes its cost function~\eqref{eqn:cost}. Denote this trajectory and its cost as $\pi^*[\x_0, \x_1]$ and $c^*[\x_0, \x_1]$ respectively:
\begin{align*}
\pi^*[\x_0, \x_1] &= \argmin\{\pi \mid\x[0] = \x_0 \wedge \x[\T^*] = \x_1\}\ c[\pi]\\
c^*[\x_0, \x_1]   &= \min\{\pi \mid\x[0] = \x_0 \wedge \x[\T^*] = \x_1\}\ c[\pi].
\end{align*}

Let us define the weighted controllability Gramian $G[t]$ as the solution of the Lyapunov equation
\[
\dot G[t] = AG[t]+G[t]A^T+BR^{-1}B^T, \quad G[0] = 0,
\]
which has the closed form expression
\begin{equation}\label{eqn:G}
G[t] = \int_0^t \exp[As]BR^{-1}B^T\exp[A^Ts]\dif s.
\end{equation}
Under the assumptions $A_{\Sigma}$ (in particular, system~\eqref{eqn:linsys} is controllable), we have that $G[t]$ is symmetric positive definite for all $t > 0$. This fact allows us to define the weighted norm $\|\cdot\|_{G^{-1}}$ for $\x \in \R^n$:
\[
\Gnorm{\x}{t} = \sqrt{\x^T G[t]^{-1} \x}.
\]

Let $\bar \x[t]$ be the zero input response of system~\eqref{eqn:linsys}, that is the solution of the differential equation
\[
\dot {\bar \x}[t] = A \bar \x[t] + \c, \quad \bar \x[0] = \x_0,
\]
which has the closed form expression
\begin{equation}\label{eqn:xbar}
\bar \x[t] = \exp[At]\x_0 + \int_0^t \exp[As]\c\dif s.
\end{equation}
Then for a fixed arrival time $\T$ the optimal control policy for the fixed-time 2BVP is given by \cite{FL-VS:95}:
\begin{equation}\label{eqn:optu}
\u[t] = R^{-1}B^T\exp[A^T(\T-t)]G[\T]^{-1}(\x_1-\bar \x[\T]),
\end{equation}
which corresponds to the minimal cost (as a function of travel time $\T$)
\begin{equation}\label{eqn:optc}
c[\T] = \T + \Gnorm{\x_1-\bar \x[\T]}{\T}^2.
\end{equation}
The optimal connection time $\T^*$ may be computed by minimizing \eqref{eqn:optc} over $\T$. The state trajectory $\x[t]$ that evolves from this control policy may be computed explicitly as:
\begin{equation}\label{eqn:exptraj}
\x[t] = \bar \x[t] + G[t]\exp[A^T(\T^*-t)]G[\T^*]^{-1}(\x_1-\bar \x[\T^*]).
\end{equation}
Let $\pi^*[\x_0, \x_1, ..., \x_K]$ denote the concatenation of the trajectories $\pi^*[\x_k, \x_{k+1}]$ between successive states $\x_0, \x_1, ..., \x_K \in \M$.
\vspace{-.1cm}
\subsection{Small-Time Characterization of the Spectrum of the Controllability Gramian}

We begin by briefly reviewing the concept of \emph{controllability indices}\footnote{See \cite[p. 431]{TK:80} or \cite[p. 150]{CTC:95} for a more detailed treatment.} for a controllable system $(A,B)$. Let $\b_k$ denote the $k$th column of $B$. Consider searching the columns of the controllability matrix
$
\C[A,B] = \begin{bmatrix} B\enspace AB\enspace \cdots\enspace A^{n-1}B\end{bmatrix}
$
from left to right for a set of $n$ linearly independent vectors. This process is well-defined for a controllable pair $(A,B)$ since $\text{rank}\left[\C[A,B]\right]=n$. The resulting set
$
\mathcal{S} = \left\{\b_1, A\b_1, \dots, A^{\nu_1 - 1}\b_1, \b_2, \dots, A^{\nu_2 - 1}\b_2, \dots, A^{\nu_m - 1}\b_m \right\}
$
defines the controllability indices $\left\{\nu_1,\dots,\nu_m\right\}$ where $\sum_{k=1}^m \nu_k = n$ and $\nu = \max_k \nu_k$ is called the \emph{controllability index} of $(A,B)$. The $\nu_k$ give a fundamental notion of how difficult a system is to control in various directions; indeed these indices are a property of the system invariant with respect to similarity transformation, e.g. permuting the columns of $B$. We may also label the vectors of $\mathcal{S}$ as $\v_1, \dots, \v_n$ in the order that they come up in $\C[A,B]$. That is, $\v_i = A^{e_i}\b_{k_i}$ and $e_i \leq e_j$ iff $i \leq j$ (note: $e_n = \nu-1$). Let $\hv_1, \dots, \hv_n$ be an orthogonalization of the $\v_i$'s so that $V = \hat V X$ where $V$ and $\hat V$ have the $\v_i$'s and $\hv_i$'s as columns respectively, and $X$ is upper triangular.

\begin{lemma}[Small-Time Gramian Asymptotics]\label{lem:STGA}
Let the eigenvalues of $G[t]$ be $\lambda_1[t] \geq \lambda_2[t] \geq \dots \geq \lambda_n[t]$. Then $\lambda_i[t] = \Theta[t^{2e_i+1}]$ as $t \rightarrow 0$ for $i = 1,\dots,n$.
\end{lemma}
\begin{proof}
We apply the Courant-Fischer Theorem:
\small\begin{equation}\label{eqn:CF}
\lambda_i = \min_{\mathcal{V}}\{\max\{R_{G[t]}[\z] \mid \z \in \mathcal{V}^*\} \mid \dim[\mathcal{V}] = n - i + 1\}
\end{equation}\normalsize
where $R_{G[t]}[\z] = \z^T G[t] \z / \|\z\|$, $\mathcal{V}$ denotes a linear subspace of $\R^n$, and $\mathcal{V}^* = \mathcal{V} \setminus \{\0\}$. Note that
\small\begin{align*}
\hv_i^T \exp[A t]B &= \sum_{j=0}^{e_i-1} \frac{t^j}{j!} \hv_i^T A^j B + \sum_{j=e_i}^\infty \frac{t^j}{j!} \hv_i^T A^j B\\
                   &= \sum_{j=e_i}^\infty \frac{t^j}{j!} \hv_i^T A^j B = \Theta[t^{e_i}] \quad \atoz,
\end{align*}\normalsize
because $\hv_i^T A^j b_k = 0$ for all $j < e_i$, $k = 1,\dots,m$ by construction. Then
\small\begin{align*}
R_{G[t]}[\hv_i] &= \int_0^t \hv_i^T\exp[A s]BR^{-1}B^T\exp[A^T s]\hv_i\dif s\\
                &= \int_0^t \Theta[s^{2e_i}] \dif s = \Theta[t^{2e_i+1}] \quad \atoz.
\end{align*}\normalsize
Making the identification $\mathcal{V} = \Span\{\hv_i,\dots,\hv_n\}$ in Equation~\eqref{eqn:CF} implies that $\lambda_i = O[t^{2e_i+1}]$; to see that $\lambda_i = \Omega[t^{2e_i+1}]$ we note that any subspace $\mathcal{V}$ of dimension $n - i + 1$ cannot satisfy $\hv_1, \dots, \hv_i \in \mathcal{V}^\perp$, as $\dim[\mathcal{V}^\perp] = i-1$.
\end{proof}

Lemma~\ref{lem:STGA} has three immediate corollaries. The first upper bounds $\|G[t]^{-1}\|$ which bounds the local cost of motion in any direction. The second relates $\|\cdot\|_{G[t]^{-1}}$ to the Euclidean norm $\|\cdot\|$ through a norm-equivalence inequality. The third is a lower bound for the determinant of $G[t]$, a result that will prove useful for estimating the volumes of reachable sets.
\begin{lemma}[Small-time minimum eigenvalue of controllability Gramian]\label{lem:normG}
Suppose that the pair $(A,B)$ has controllability index $\nu$, then $\lambda_n[G[t]] = \Theta[t^{2\nu-1}]$ as $t\to0$, or, equivalently, $\|G[t]^{-1}\| = \Theta[t^{-2\nu+1}]$.
\end{lemma}
\begin{lemma}[Norm Equivalence]\label{lem:normEQ}
Suppose that the pair $(A,B)$ has controllability index $\nu$, and consider the Cholesky factorization $G[t] = L[t]L[t]^T$. Then for $\x \in \R^n$, $\|\x\|_{G[t]^{-1}} = \|L[t]^{-1}\x\|$ and
\[
\frac{\|\x\|}{\|L[t]\|} \leq \|\x\|_{G[t]^{-1}} \leq \|L[t]^{-1}\|\|\x\|
\]
where $\|L[t]\| = \Theta[t^{-1/2}]$ and $\|L[t]^{-1}\| = \Theta[t^{-\nu+1/2}]$.
\end{lemma}
\begin{lemma}[Small-time determinant of controllability Gramian]\label{lem:detG}
Suppose that the pair $(A,B)$ has controllability indices $\left\{\nu_1,\dots,\nu_m\right\}$, then $\det[G[t]] = \Theta[t^D]$ as $t\to0$ where $D = \sum_{k=1}^m \nu_k^2$.
\end{lemma}

\section{Probabilistic Exhaustivity}\label{sec:exhaust}

In this section we prove a key result characterizing random sampling schemes for the LQDMP problem: \emph{any} feasible trajectory through the configuration space $\M$ is ``traced" arbitrarily well by connecting randomly distributed points from a sufficiently large sample set covering the configuration space. We will refer to this property as probabilistic exhaustivity. The same notion of probabilistic exhaustivity (clearly much stronger than the usual notion of probabilistic completeness) was introduced in the related paper \cite{ES-LJ-MP:15a} in the context of DMP for \emph{driftless} systems. The result proven in that work does not carry over to the drift case as it relies on the metric inequality to bound the cost of approximate paths; the drift case lacks the control symmetry to make such estimates. Thus in order to prove probabilistic exhaustivity in the case of linear affine systems, we first provide a result analogous to the metric inequality characterizing the effect that perturbations of the endpoints of a path have on its cost and state trajectory. The idea, then, is that tracing waypoints may be selected as small perturbations of points along the trajectory to be approximated, provided the sample density is high enough.

\begin{lemma}[Fixed-Time Local Trajectory Approximation]\label{lem:ball2ball}
Let $\x_0, \x_1 \in \M$, $\x_0 \neq \x_1$, $\pi = \pi^*[\x_0,\x_1]=(\x, \u, \T^*)$, and denote $c = c[\pi]$. Consider bounded start and end state perturbations $\dx_0, \dx_1 \in \R^n$ such that $\max\{\Gnorm{\dx_0}{\T^*}, \Gnorm{\dx_1}{\T^*}\} \leq \eta \sqrt{c}$. Let $\sigma = \pi^*[\x_0 + \dx_0,\x_1 + \dx_1] =  (\y, \v, \tilde\T^*)$ be the optimal trajectory between the perturbed endpoints. Then for $\x_0, \x_1$ such that $c$ is sufficiently small, we have the cost bound
\[
c[\sigma] \leq c[\pi]\left(1 + 4\eta + O[\eta^2 + \eta c]\right).
\]
Additionally we may bound the geometric extent of $\sigma$:
\[
\|\y[t]-\x_0\| = O\left[c(\|\x_0\| + \eta + 1)\right]
\]
for $t \in [0, \tilde\T^*]$.
\end{lemma}
\begin{proof}
Since $\sigma$ is the optimum, regardless of the value of $\tilde\T^*$ we have the upper bound
\small\begin{align*}
c[\sigma] \leq \T^* + \Gnorm{\x_1 - \xb[\T^*] + \dx_1 - \exp[A\T^*]\dx_0}{\T^*}^2
\end{align*}\normalsize
which expands as
\small\begin{align*}
c[\sigma] - c &\leq 2(\x_1 - \xb[\T^*])^T G[\T^*]^{-1} (\dx_1 - \exp[A\T^*]\dx_0)\\
              &\qquad + \Gnorm{\dx_1 - \exp[A\T^*]\dx_0}{\T^*}^2\\
          &\!\!\!\!\!\leq 2\Gnorm{\x_1 - \xb[\T^*]}{\T^*} \Gnorm{\dx_1 - \exp[A\T^*]\dx_0}{\T^*}\\
          &\qquad + O[\eta^2 c]\\
          &\!\!\!\!\!\leq 2\sqrt{c-\T^*} \left(\eta\sqrt{c} + \eta\sqrt{c}(1 + O[\T^*])\right) + O[\eta^2 c]\\
          &\!\!\!\!\!\leq 4\eta c (1 + O[\eta + c]),
\end{align*}\normalsize
where in the second line we have applied the (weighted) Cauchy-Schwarz inequality, and, in the last line, the fact that $\T^* \leq c$.

To bound $\|\y[t] - \x_0\|$, we first apply Cauchy-Schwarz:
\small\begin{align*}
\int_{0}^{\tilde\T^*} \|\v[t]\|\dif t & \leq \left(\tilde\T^* \int_{0}^{\tilde\T^*} \|\v[t]\|^2\dif t \right)^{1/2} \\
                                      & \leq \left(\frac{\tilde\T^*}{\lambda_{\min}(R)} \int_{0}^{\tilde\T^*} \v[t]^T R \v[t]\dif t \right)^{1/2} \\
                                      & = \left(\frac{\tilde\T^*}{\lambda_{\min}(R)} (c[\sigma] - \tilde\T^*) \right)^{1/2} \leq \frac{c[\sigma]}{2\sqrt{\lambda_{\min}(R)}}.
\end{align*}\normalsize
Then integrating the system dynamics \eqref{eqn:linsys} yields
\small\begin{align*}
&\|\y[t] - \x_0\| = \left\|(e^{At} - I)\x_0 + e^{At}\dx_0 + \int_0^t e^{As}\c\dif s\right. \\
                 &\quad \qquad \qquad \qquad \qquad \qquad \qquad \left.+ \int_0^t e^{A(t-s)}B\v(s)\dif s \right\|\\
                 &= O\left[\tilde\T^* \|\x_0\| + \T^{*^{1/2}}\Gnorm{\dx_0}{\T^*} + \tilde\T^* + \int_{0}^{\tilde\T^*} \|\v[t]\|\dif t\right]\\
                 &= O[c(\|\x_0\| + \eta + 1)],
\end{align*}\normalsize
making use of a norm equivalence bound (Lemma~\ref{lem:normEQ}) in line two. We note that the asymptotic constants depend only on the fixed system dynamics $\Sigma = (A, B, \c, R)$.
\end{proof}

Motivated by Lemma~\ref{lem:ball2ball}, we define the perturbation ball
\[
\Delta[\x, \T, r] = \left\{\z \mid \Gnorm{\x-\z}{\T} \leq r\right\}.
\]
This set represents perturbations of $\x$ with limited effects on both incoming and outgoing trajectories (depending on whether a point is viewed as an end state or start state perturbation respectively). We note that since $\|\cdot\|_{G[t]^{-1}}$ is decreasing as $t$ increases, we have
\begin{equation}\label{eqn:ballcontain}
\Delta[\x, \T_1, r] \subset \Delta[\x, \T_2, r] \ \text{ for } \T_1 \leq \T_2.
\end{equation}
To understand how often sample points of a planning algorithm will lie within $\Delta[\x, \T, r]$, we lower bound its volume.
\begin{remark}[Bounding Perturbation Ball Volume]\label{rem:ballvol}
The inequality $\Gnorm{\x-\z}{\T} \leq r$ defines an ellipse with volume
\[
\mu[\Delta[\x, \T, r]] = r^n \zeta_n \sqrt{\det[G[\T]]}
\]
where $\zeta_n$ denotes the volume of the unit ball in $\R^n$. Given our asymptotic characterization of $\det[G[\T]]$ in Lemma~\ref{lem:detG}, there is a threshold $\T_\mu > 0$ and constant $C_\mu > 0$ such that
\[
\mu[\Delta[\x, \T, r]] \geq C_\mu r^n \T^{D/2}
\]
for all $\T < \T_\mu$.
\end{remark}

To ensure that the $\T^{D/2}$ term above does not become vanishingly small in application with Lemma~\ref{lem:ball2ball}, we also lower bound connection time in terms of connection cost. 

\begin{lemma}[Optimal Cost/Time Breakdown]\label{lem:costtime}
Let $\x_0, \x_1 \in \M$, $\x_0 \neq \x_1$, $\pi = \pi^*[\x_0,\x_1]=(\x, \u, \T^*)$, and denote $c = c[\pi]$. For $\x_0, \x_1$ such that $c$ is sufficiently small,
$
\tau^* \geq \frac{c}{3}.
$
\end{lemma}
\begin{proof}
For fixed $\x_0 \neq \x_1$, consider the series expansion of the control effort term $\Gnorm{\x_1-\bar \x[\T]}{\T}^2$ about $\T = 0$. We claim that $\Gnorm{\x_1-\bar \x[\T]}{\T}^2 = a\T^{-e} + O[\T^{-e+1}]$ for some $1 \leq e \leq 2\nu - 1$. The fact that $e \geq 1$ follows from the fact that $\x_1-\bar \x[\T] = (\x_1 - \x_0) + O[\T]$ has a nonzero zeroth order term and $\lambda_{\min}[G[\T]^{-1}] = \Theta[\T^{-1}]$ as a consequence of Lemma~\ref{lem:STGA}. We note that for general $\x_0, \x_1$, however, it is almost certain that $(\x_1 - \x_0)$ or one of the low-order (in $\T$) terms of $\x_1-\bar \x[\T]$ will have a component along the maximal eigenvector of $G[\T]^{-1}$, which may result in a series expansion term with $e$ up to $2\nu - 1$. Then
\[
c[\T] = \T + a\T^{-e} + O[\T^{-e+1}]
\]
is maximized at $\T^* = (ae)^{1/(1+e)} + o(1)$ and we compute the ratio
\[
\T^* / c = (1+1/e)^{-1}(1+o(1))
\]
as $c \to 0$. The dominant term in the asymptotics is smallest when $e=1$ (corresponding to $\T^* / c \approx 1/2$); in particular we have $\T^* / c \geq 1/3$ for $\x_0, \x_1$ with $c$ sufficiently small.\footnote{We note that this bound does not depend on the actual values of $\x_0$ and $\x_1$ (in particular the constant $a$), but only their optimal connection cost.}
\end{proof}

Lemma~\ref{lem:ball2ball} is a statement about local trajectory approximation. We now define what it means for a series of states to closely approximate a given global trajectory. Let $\pi = (\x, \u, T_\pi)$ be a dynamically feasible trajectory. Given a set of waypoints $\{\y_k\}_{k=0}^K \subset \M$, we associate the trajectory $\sigma = \pi^*[\y_0,\dots,\y_K]=(\y,\v,T_\sigma)$. We consider the $\{\y_k\}$ to \emph{$(\eps, r, p)$-trace} the trajectory $\pi$ if: (a) the cost of $\sigma$ is bounded as $c[\sigma] \leq (1+\eps)c[\pi]$, (b) $c^*[\y_k,\y_{k+1}] \leq r$ for all $k$, and (c) the maximum distance from any point of $\y$ to $\x$ is no more than $p$, i.e.
$
\max_{t\in[0,T_\sigma]}\left(\min_{s\in[0,T_\pi]} \|\y[t] - \x[s]\|\right) \leq p.
$
The combination of these three properties is what makes $\sigma$, if approximating a near-globally-optimal trajectory $\pi$, amenable to recovery by the path planning algorithms we propose in the next section. In particular, (b) ensures that $\sigma$ is the concatenation of uniformly local connections.
In Theorem~\ref{thm:pathtracing} we show that suitable waypoints may be found with high probability as a subset of a set of randomly sampled nodes, the proof of which requires the following two technical lemmas lower bounding the probability that a sample set will provide adequate coverage around a trajectory of interest.
Let $\SF[N]$ denote a set of $N$ points sampled independently and identically from the uniform distribution on $\Mfree$.
\begin{lemma}[Lemma IV.3, \cite{ES-LJ-MP:15a}]\label{lem:samplesmall} Fix $N \in \N$, $\alpha \in (0,1)$, and let $S_0, \dots, S_K$ be disjoint subsets of $\Mfree$ with
\[
\mu[S_k] = \mu[S_1] \geq (2+\log(1/\alpha))e^2\left(\frac{1}{N}\right) \mu[\Mfree],
\]
for each $k$. Let $V = \SF[N]$; then the probability that more than an $\alpha$ fraction of the sets $S_k$ contain no point of $V$ is bounded as:
\[
\p{\#\{k \in \{0,\dots,K\} : S_k \cap V = \emptyset\} \geq \alpha K} \leq 2e^{-\alpha  K}.
\]
\end{lemma}

\begin{lemma}[Lemma IV.4, \cite{ES-LJ-MP:15a}]\label{lem:samplebig} Fix $N \in \N$ and let $T_0, \dots, T_K$ be subsets of $\Mfree$, possibly overlapping, with
\[
\mu[T_k] = \mu[T_1] \geq \kappa\left(\frac{\log N}{N} \right)\mu[\Mfree]
\]
for each $k$ and some constant $\kappa > 0$. Let $V = \SF[N]$; then the probability that there exists a $T_k$ that does not contain a point of $V$ is bounded as:
\[
\p{\bigvee_{k=0}^K \{T_k \cap V = \emptyset\}} \leq KN^{-\kappa}.
\]
\end{lemma}

The proofs of these two lemmas may be found in our related work \cite{ES-LJ-MP:15a}. As in that work and \cite{LJ-ES-AC-ea:15}, our approach here for proving probabilistic exhaustivity proceeds by tiling the span of a path to be traced with two sequences of concentric perturbation balls – a sequence of ``small'' balls and a sequence of ``large'' balls. With high probability, all but a tiny $\alpha$ fraction of the small balls will contain a point from the sample set (Lemma~\ref{lem:samplesmall}), and for any small balls that do not contain such a point we ensure that the corresponding large ball does (Lemma~\ref{lem:samplebig}). We take these points as a sequence of waypoints which tightly follows the reference path with few exceptions, and never has a gap over any section of the reference path when it does deviate further.

\begin{theorem}[Probabilistic exhaustivity]\label{thm:pathtracing}
Let $\Sigma$ be a system satisfying the assumptions $A_\Sigma$ and suppose $\pi=(\x,\u,T)$ is a dynamically feasible trajectory with strong $\delta$-clearance, $\delta > 0$. Let $N\in\N$, $\eps > 0$, and consider a set of sample nodes $V = \{\x[0]\} \cup \SF[N]$. Define $\tD = (n+D)/2$, $C_{\Sigma, \Mfree} = \left(C_\mu^{-1} \tD^{-1}6^{n+D/2}2^{n/2}\mu[\Mfree]\right)^{1/\tD}$ and consider the event $E_N$ that there exist waypoints $\{\y_k\}_{k=0}^K \subset V$ which $(\eps, r_N, p_N)$-trace $\pi$, where
\[
r_N\!=\!(1+\eta)^{1/\tD} C_{\Sigma, \Mfree} \left(\frac{\log N}{N}\right)^{1/\tD}\!,
\]
for a free parameter $\eta \geq 0$, and $p_N = C_p r_N$ for some constant $C_p$.
Then, as $\Nti$, the probability that no such waypoint set exists is asymptotically bounded as
\[
1 - \p{E_N} = O\left(N^{-\eta/\tD}\log^{-1/\tD} N\right).
\]
\end{theorem}
\begin{proof}
Note that in the case $c[\pi] = 0$ we may pick $\y_0 = \x[0]$ to be the only waypoint and the result is trivial. Therefore assume $c[\pi] > 0$. Make the identification $\alpha = \eps/4$, $\beta = \eps/2$ and fix $N$ sufficiently large so that $r_N/6 \leq \T_\mu$ and also:
\begin{equation}\label{eqn:smallugly}
\log N \geq \beta^{-n}\tD (2-\log(\alpha))e^2/(1+\eta).
\end{equation}
Take $\x[t_k]$ to be points spaced along $\pi$ at cost intervals $r_N/2$; more precisely let $t_0 = 0$, and for $k=1,2,\dots$ consider
\[
t_k \!=\! \min\left\{t \in (t_k, T) \mid c^*[\x[t_{k-1}], \x[t]] \geq r_N/2\right\}.
\]
Let $K$ be the first $k$ for which the set is empty; take $t_K = T$. Note that by construction, we have $K \leq \lceil 2c[\pi]/r_N \rceil$.

We consider the sets $T_k = \Delta[\x[t_k],r_N/6,(1/6)\sqrt{r_N/2}]$ and $S_k = \Delta[\x[t_k],r_N/6,(\beta/6) \sqrt{r_N/2}]$. In particular the time $r_N/6$ here is chosen so that, by Lemma~\ref{lem:costtime}, the optimal connection times between the $\x_k$ satisfy $\T^*[\x[t_{k-1}], \x[t_k]] \geq c^*[\x[t_{k-1}], \x[t_k]]/3 = r_N/6$. Applying the ball containment property \eqref{eqn:ballcontain} this means that for any such $\T^*$, $\Delta[\x[t_k],r_N/6,\rho] \subset \Delta[\x[t_k],\T^*,\rho]$ for $\rho = \sqrt{r_N/2}$ or $(\beta/4) \sqrt{r_N/2}$. From Remark~\ref{rem:ballvol} and our choice of $r_N$ we have the volume bound
\begin{equation}\label{eqn:bigvol}
\mu(T_k) \geq \mu[\Mfree] \left(\frac{1+\eta}{\tD}\right)\left(\frac{\log N}{N}\right),
\end{equation}
and similarly
\begin{equation}\label{eqn:smallvol}
\mu(S_k) \geq \beta^n \mu[\Mfree] \left(\frac{1+\eta}{\tD}\right)\left(\frac{\log N}{N}\right),
\end{equation}
for each $k$. Combining equation~\eqref{eqn:bigvol} and Lemma~\ref{lem:samplebig}, we have that the probability that there exists a $T_k$ that does not contain a sample point (i.e. $T_k \cap V = \emptyset$) is bounded as:
\begin{equation*}
\p{\bigvee_{m=0}^K \{T_k \cap V = \emptyset\}} \leq K N^{-(1+\eta)/\tD}.
\end{equation*}
We note that the $S_k$ are disjoint (as long as $\eps < 2\sqrt{6}$) since $c^*[\x[t_{k-1}], \x[t_k]] = r_N/2$ implies $\Gnorm{\x[t_{k-1}] - \x[t_k]]}{r_N/6} \geq \sqrt{r_N/3}$. Then we may combine equations~\eqref{eqn:smallugly} and \eqref{eqn:smallvol}, which together imply that the $S_k$ satisfy the condition of Lemma~\ref{lem:samplesmall}, to see that the probability that more than an $\alpha$ fraction of the $S_k$ do not contain a sample point is bounded as:
\begin{equation*}
\p{\#\{k \in \{0,\dots,K\} : S_k \cap V = \emptyset\} \geq \alpha K} \leq 2e^{-\alpha  K}.
\end{equation*}
Now, as long as neither of these possibilities holds (i.e. if every $T_k$ and at least a $(1-\alpha)$ fraction of the $S_k$ contains a point of $V$), we will show that the existence of suitable waypoints $\{\y_k\}_{k=0}^K \subset V$ is guaranteed. In that case then we may union bound the probability  of failure:
\begin{align*}
1 - \p{E_N} &\leq K N^{-(1+\eta)/\tD} + 2e^{-\alpha  K}\\
&= O\left[N^{-\eta/\tD}\log^{-1/\tD} N\right],
\end{align*}
as $\Nti$ (the first term dominates asymptotically), where we have used the fact that $K \leq \lceil 2c[\pi]/r_N \rceil = O[(N/\log N)^{1/\tD}].$

Suppose that every $T_k$ and at least a $(1-\alpha)$ fraction of the $S_k$ contains a point of $V$. Choose points $\{\y_k\} \subset V$ accordingly: within $S_k$ if possible, and within $T_k$ otherwise. We may apply Lemma~\ref{lem:ball2ball} to verify that these points $(\eps, r_N, p_N)$-trace $\x$. For $\y_k \in S_k, \y_{k+1} \in S_{k+1}$ we have:
\begin{align*}
c^*[\y_{k-1}, \y_k] &\leq (1+\beta) c^*[\x[t_{k-1}], \x[t_k]]\\
                    &= (1+\eps/2)(r_N/2).
\end{align*}
Since all but a $2\alpha = \eps/2$ fraction of successive points $\y_{k-1}, \y_k$ must both be in $S$ sets and obey the above cost bound, and the remaining pairs satisfy the analogous bound for $T$ sets (with $2$ instead of $(1+\eps/2)$), the total cost of $\pi^*[\y_0,\dots,\y_K] = (\y, \v, T_\sigma)$ is bounded above by $(1+\eps)c[\pi]$. We also have $c^*[\y_{k-1}, \y_k] \leq r_N$ for all $k$. The maximum Euclidean distance from any point of $\y$ (say, on the segment $\pi^*[\y_k, \y_{k+1}]$) to $\x$ is bounded above by its distance to $\x[t_k]$, which by Lemma~\ref{lem:ball2ball} is $O[r_N(\|\x_0\| + 1)] = O[r_N]$ as $N \to \infty$ since $\|\x[t]\|$ achieves some fixed maximum over $[0,T]$.
\end{proof}

\section{ \DFMT Algorithm}\label{sec:algo}

The algorithm presented here is based on \FMT, from the recent work of \cite{LJ-ES-AC-ea:15}, which can be thought of as an accelerated version of \PRM \cite{SK-EF:11}. Briefly, \PRM first samples all the vertices, then constructs a fully \emph{locally} connected graph, and then performs shortest path search (e.g., Dijkstra's algorithm) on the graph to obtain a solution. \FMT also samples all vertices first, but instead of a graph, lazily builds a tree \emph{via dynamic programming}  that very closely approximates the shortest-path tree for \PRM, but saves a multiplicative factor of $O[\log(n)]$ collision-checks by not constructing the full graph. The algorithm given by Algorithm \ref{fmtalg}, \DFMT, is not fundamentally different from the original \FMT algorithm, but mainly changes what ``local" means under differential constraints (similar to \cite{ES-LJ-MP:15a}, but now with drift). One more difference of \DFMT presented here, even from the algorithm in \cite{ES-LJ-MP:15a}, is that the edges are now directed, reflecting the fundamental asymmetry of differential constraints with drift.

Specifically, define the fixed-time forward-reachable and backwards-reachable sets respectively:
\begin{align*}
R^+[\x, r] &= \{\x' \in \M \mid c^*[\x, \x'] < r\} \\
R^-[\x, r] &= \{\x' \in \M \mid c^*[\x', \x] < r\}.
\end{align*}
Membership in either reachable set may be checked by minimizing the explicit cost function~\eqref{eqn:optc} over travel time.
The set of samples to check for membership may be pruned by considering the form of $G[t]$, as suggested in \cite{DW-JvdB:13}. Let $\texttt{CollisionFree}[\x_1, \x_2]$ denote the boolean function which returns true if and only if $\pi^*[\x_1, \x_2]$ lies within $\Mfree$. Given a set of vertices $V$, a state $\x \in \M$, and a cost threshold $r > 0$, let $\texttt{Near}^\pm[V, \x, r] = V \cap R^\pm[\x, r]$. Let $(\x_1,\x_2)$ denote the directed edge corresponding to $\pi^*[\x_1, \x_2]$ with edge weight $c^*[\x_1, \x_2]$. Given a directed graph $G = (V, E)$, where $V$ is the vertex set and $E$ is the edge set, and a vertex $\x \in V$, let $\texttt{Cost}[\x, G]$ be the function that returns the cost of the shortest (directed) path in the graph $G$ between the vertices $\xinit$ and $\x$. Let $\texttt{Path}[\x, G]$ be the function that returns the path achieving that cost. The \DFMT algorithm is given in Algorithm \ref{fmtalg}. The algorithm uses two mutually exclusive sets, namely  $H$ and $W$. The \emph{unexplored} set $W$ stores all samples in the sample set $V$ that have not yet been considered for addition to the tree of paths. The \emph{wavefront} set $H$, on the other hand, tracks in sorted order (by cost from the root) only those nodes which have already been added to the tree that are near enough to tree leaves to actually form better connections. A detailed description of the algorithm would parallel the one provided in \cite{LJ-ES-AC-ea:15} and is omitted due to space limitations, we refer the interested reader to  \cite{LJ-ES-AC-ea:15}. An extension of \PRM, which we denote by \DPRM, may also be defined in a straightforward manner as in \cite{ES-LJ-MP:15a}, although we omit the full description here. Briefly, \DPRM searches the graph of all local collision-free connections that appear in any $\texttt{Near}$ set (as opposed to the tree subgraph constructed by \DFMT) for the least cost trajectory.

\begin{algorithm}[!ht]
\caption{Differential Fast Marching Tree (\DFMT)}
\label{fmtalg}
\algsetup{linenodelimiter=}
\begin{algorithmic}[1]
\STATE $V \leftarrow \{\xinit\} \cup \texttt{SampleFree}[N]$; $E \leftarrow \emptyset$
\STATE $W \leftarrow V \setminus \{\xinit\}$; $H \leftarrow \{\xinit\}$; $\z \leftarrow \xinit$
\WHILE{$\z \notin \Mgoal$}
\STATE $H_{\text{new}} \leftarrow \emptyset$ 
\STATE $X_{\text{near}} = \texttt{Near}^+[V \setminus \{\z\}, \z, r_N] \cap W$
\FOR{$\x \in X_{\text{near}}$}
\STATE $Y_{\text{near}} \leftarrow \texttt{Near}^-[V \setminus \{\x\}, \x, r_N] \cap H$ \label{line:intersect}
\STATE $\y_{\text{min}} \leftarrow \arg\min_{\y \in Y_{\text{near}}}\{\texttt{Cost}[\y, T = (V, E)] \!+\! c^*[\y,\x]\}$
\IF{$\texttt{CollisionFree}[\y_{\text{min}}, \x]$} \label{line:locon}
\STATE $E \leftarrow E \cup \{(\y_{\text{min}}, \x)\}$
\STATE $H_{\text{new}} \leftarrow H_{\text{new}} \cup \{\x\}$ \label{alg:H_1}
\STATE $W \leftarrow W \setminus \{\x\}$
\ENDIF
\ENDFOR
\STATE $H \leftarrow (H \cup H_{\text{new}}) \setminus \{\z\}$ \label{alg:H_2} 
\IF{$H = \emptyset$}
\RETURN \texttt{Algorithm Failure}
\ENDIF
\STATE $\z \leftarrow \arg\min_{\y \in H}\texttt{Cost}[\y, T = (V, E)]$
\ENDWHILE
\RETURN $\texttt{Path}[\z, T = (V, E)]$
\end{algorithmic}
\end{algorithm}

\vspace{-.1cm}
\section{Asymptotic Optimality of  \DFMT}\label{sec:AO}
In this section, we state the asymptotic optimality of \DFMT, for which the asymptotic optimality of \DPRM is a corollary. We note that in contrast to the work required to establish probabilistic exhaustivity for this class of differentially constrained systems and cost functions, the argument that \DFMT recovers paths at least as good as any waypoint-traced trajectory is essentially equivalent to the proofs presented in \cite{ES-LJ-MP:15a} and \cite{LJ-ES-AC-ea:15}. That is, the idea that \DFMT (or \DPRM) can connect closely spaced sample points at a resolution sufficiently fine that every connection takes place away from the influence of the obstacle set is not a feature specific to the LQDMP problem. Thus we state the following theorem and, in the interest of brevity, refer the reader to the proofs of Theorems VI.1 and VI.2 presented in \cite{ES-LJ-MP:15a}. We note that the following optimality result for \DFMT also provides a \emph{convergence rate} bound, but to avoid confusion we emphasize that this bound is given in terms of sample size $N$. For a discussion of how sample size relates to run time for \FMT-style algorithms see \cite{LJ-ES-AC-ea:15}.

\begin{theorem}[\DFMT asymptotic optimality]\label{thm:DFMTAO}Let $(\Sigma, \Mfree, \xinit, \Mgoal)$ be a trajectory planning problem satisfying the assumptions $A_\Sigma$ and with $\Mgoal$ $\xi$-regular, such that there exists an optimal path $\pi^*$ with weak $\delta$-clearance for some $\delta > 0$. Let $c_N$ denote the cost of the path returned by \DFMT with $N$ vertices using the cost threshold
\[
r_N\!=\!(1+\eta)^{1/\tD} C_{\Sigma, \Mfree} \left(\frac{\log N}{N}\right)^{1/\tD}\!,
\]
where $\tD = (n+D)/2$, $C_{\Sigma, \Mfree} = \left(C_\mu^{-1} \tD^{-1}6^{n+D/2}2^{n/2}\right.\\\left.\mu[\Mfree]\right)^{1/\tD}$, and $\eta \geq 0$ is an implementation-specific parameter. Then for fixed $\eps > 0$,
\[
\p{c_N > (1+\eps)c(\pi)} = O\left(N^{-\eta/\tD}\log^{-1/\tD} N\right).
\]
\end{theorem}

\section{Numerical Experiments}\label{sec:sims}
\begin{figure}[!htbp]
 \centering
 \vspace{3mm}
 \subfigure{\includegraphics[width=0.3\textwidth]{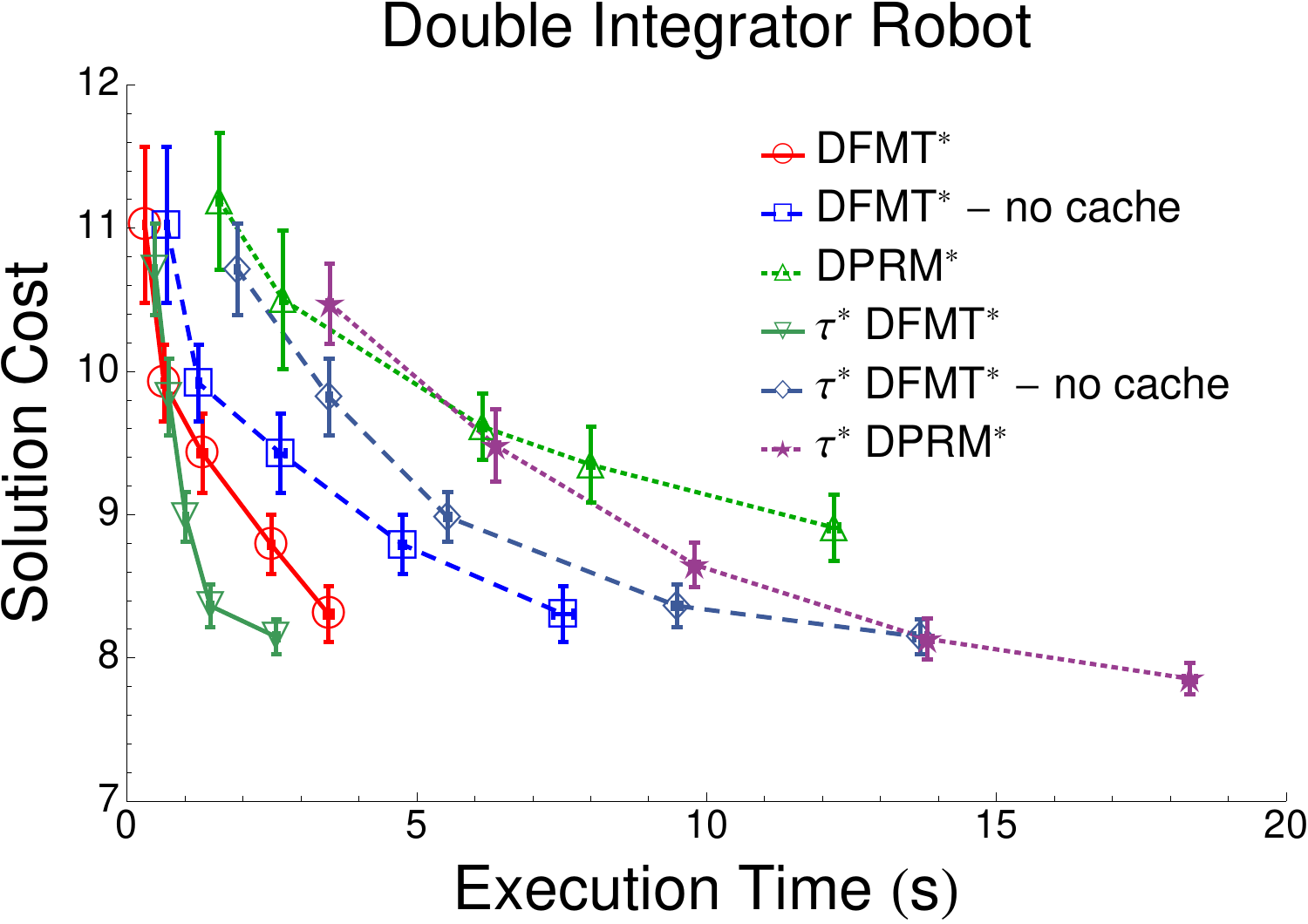}}
 \subfigure{\includegraphics[width=0.3\textwidth]{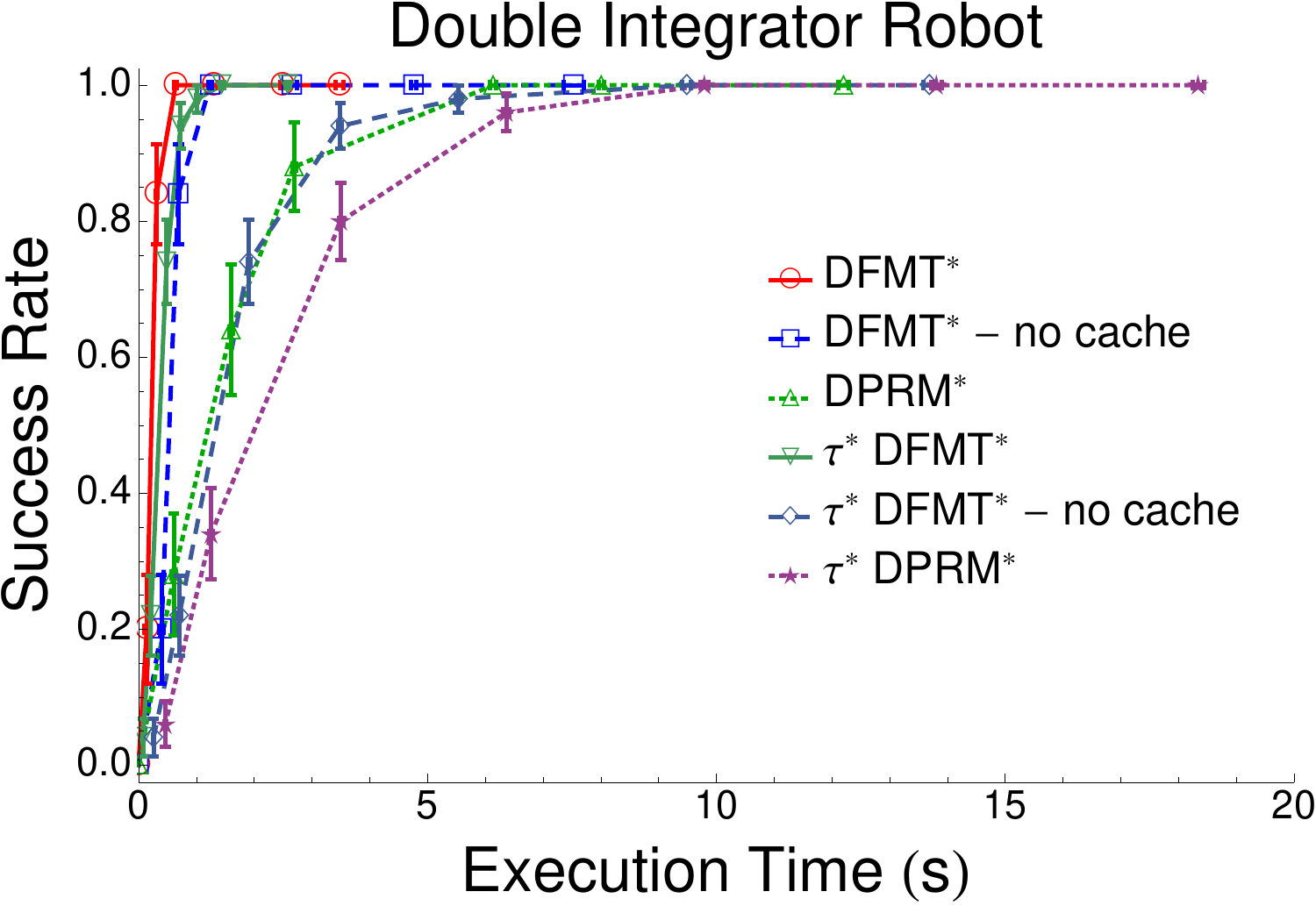}} \vspace{-3mm}
 \subfigure{\includegraphics[width=0.3\textwidth]{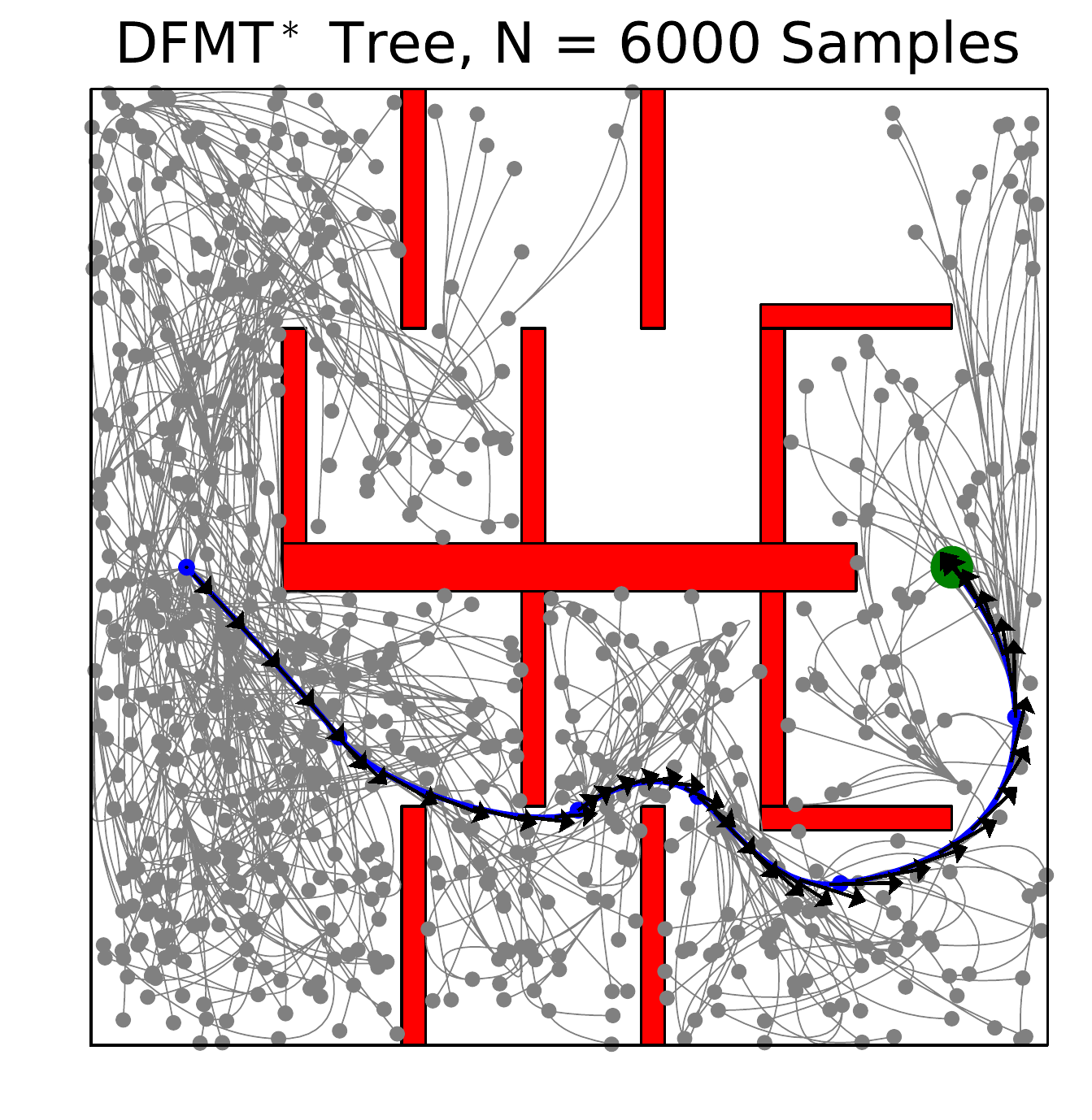}} 
 \caption{Top: Simulation results for the double integrator system with a maze obstacle set. The error bars in each axis represent plus and minus one standard error of the mean for a fixed sample size $n$.
 Bottom: Example \DFMT tree for $n=6000$. The position of the the feasible trajectory returned is highlighted in blue with velocity denoted by arrows.} \vspace{-5mm}
\label{fig:all}
\end{figure}
The \DFMT and \DPRM algorithms were implemented in Julia and run using a Unix operating system with a 2.0 GHz processor and 8 GB of RAM.
We tested \DFMT and \DPRM on the double integrator system, a standard LQDMP formulation as studied in \cite{DW-JvdB:13}. We also implemented variants of \DFMT and \DPRM where the local connection cost is computed with respect to a fixed time $\T$, instead of optimizing $c^*[\T]$ over all arrival times. This is less computationally intensive than searching for the optimal $\T^*$, and can be proven asymptotically optimal as well using a similar probabilistic exhaustivity approach. The intuition is that for a fixed cost radius, the algorithm is searching over a ``donut''  -- the band of states in which the fixed $\T$ is a valid approximation for $\T^*$ -- instead of a ``ball.'' Provided the sample count $N$ is sufficiently large, this donut has enough volume to support the desired tracing connections. We denote the optimal-connection algorithms presented in this paper $\T^*$ \DFMT and $\T^*$ \DPRM to differentiate the two approaches. The simulation results are summarized in Figure~\ref{fig:all}. A maze was used for $\Mobs$, and our algorithm implementations were run 50 times each on sample sizes up to $N=12000$ for fixed $\T$ and $N=6000$ for the $\T^*$ variants. We plot results for both versions of \DFMT run with and without a cache of near neighbor sets and local connection costs; as discussed in \cite{ES-LJ-MP:15a} this information, which does not depend on the problem-specific obstacle configuration, may be precomputed for batch-processing algorithms such as \DFMT and \DPRM\ -- the price to pay is a moderate increase in memory requirements. We see that the extra time for optimizing over local connection duration $\T$ is significant (\DFMT\ -- no cache vs. $\T^*$ \DFMT\ -- no cache), but may be mitigated by precomputation (\DFMT  vs. $\T^*$ \DFMT).

\vspace{-0.1truecm}
\section{Discussion and Conclusions}\label{sec:disc}
\vspace{-0.1truecm}
In this paper we have provided a thorough and rigorous theoretical framework to assess optimality guarantees of sampling-based algorithms for linear affine systems with a mixed time/energy cost function. In particular, we leveraged the study of small-cost perturbations to show that optimum-approximating waypoints may be found among randomly sampled state sets with high probability. We applied this analysis to design and theoretically validate an asymptotically optimal algorithm, \DFMT, for the LQDMP problem.

Although this work is nominally limited to linear affine drift systems, it not only provides a good model for many real systems, but is a crucial first step towards modelling nonlinear systems as well. Indeed, since \DFMT can be applied to a nonlinear system by linearizing the dynamics, an important next step will be to assess the theoretical guarantees of \DFMT applied to such a linearized approximation. Given suitable smoothness assumptions on the system linearization (sufficiently powerful, but still encompassing a useful class of dynamics), it seems likely that the perturbation analysis and probabilistic exhaustivity may follow identically to that presented in this paper, up to an additional term quantifying the ``perturbation on the perturbation.'' A similar linearization approach has been experimentally validated by Kinodynamic \RRTs in \cite{DW-JvdB:13}. We are also note the similarities in analysis evident between the linear affine drift systems studied in this paper and the (possibly non-linear) control-affine driftless systems studied in \cite{ES-LJ-MP:15a}. In particular, the parallel notions of controllable/bracket-generating systems and controllability index/Hausdorff dimension give hope that a unifying theory for non-linear systems with drift may be achieved.
There are a number of additional directions open for further research. In particular,  we plan to deploy \DFMT on robotic platforms, specifically helicopters and floating platforms emulating the dynamics of spacecraft.  Also,  it is of interest to study a bidirectional version of \DFMT. Finally, it is of interest to devise strategies whereby the radius tuning parameter is self regulating, with the objective of making the algorithm parameter-free.
\vspace{-0.1truecm}

\renewcommand*{\bibfont}{\footnotesize}
\bibliographystyle{IEEEtran}
\bibliography{../../../bib/alias,../../../bib/main}
\newpage
\end{document}

%% file: preamble.tex
\usepackage{amsmath}
\usepackage{amssymb}
\usepackage{graphicx}
\usepackage{comment,xspace}
\usepackage{hyperref}
\usepackage{fancybox}
\usepackage{xspace}

\newtheorem{theorem}{Theorem}[section]

\newtheorem{lemma}[theorem]{Lemma}

\newtheorem{remark}[theorem]{Remark}

\newcommand{\eps}{\varepsilon}

\newcommand{\Span}{\operatorname{span}}

\newcommand{\argmin}{\operatornamewithlimits{argmin}}

\newcommand{\atoz}{\text{as } t \to 0}

\newcommand{\p}[1]{\mbox{$\mathbb{P}\left[#1\right]$}}

\newcommand{\x}{\mathbf{x}}
\newcommand{\y}{\mathbf{y}}
\newcommand{\z}{\mathbf{z}}
\newcommand{\hv}{\mathbf{\hat v}}
\newcommand{\T}{\tau}
\newcommand{\C}{\mathcal{C}}
\renewcommand{\b}{\mathbf{b}}
\renewcommand{\c}{\mathbf{c}}
\renewcommand{\u}{\mathbf{u}}
\renewcommand{\v}{\mathbf{v}}

\newcommand{\xb}{\bar\x}
\newcommand{\dx}{\boldsymbol{\delta}\x}

\newcommand{\Gnorm}[2]{\left\|#1\right\|_{G[#2]^{-1}}}

\newcommand{\M}{\mathcal{M}}
\newcommand{\U}{\mathcal{U}}
\newcommand{\Mobs}{\mathcal{M}_{\text{obs}}}
\newcommand{\Mfree}{\mathcal{M}_{\text{free}}}
\newcommand{\Mgoal}{\mathcal{M}_{\text{goal}}}

\newcommand{\cl}{\operatorname{cl}}

\newcommand{\SF}{\texttt{SampleFree}}
\newcommand{\R}{\mathbb{R}}
\newcommand{\N}{\mathbb{N}}
\newcommand{\0}{\mathbf{0}}

\newcommand{\xinit}{\x_{\text{init}}}

\newcommand{\Nti}{N\to\infty}

\newcommand{\DFMT}{DFMT$^*$\xspace}
\newcommand{\DPRM}{DPRM$^*$\xspace}
\newcommand{\FMT}{FMT$^*$\xspace}
\newcommand{\PRM}{PRM$^*$\xspace}

\newcommand{\RRTs}{RRT$^*$\xspace}
\newcommand{\tD}{{\tilde D}}

\newcommand*\dif{\mathop{}\!\mathrm{d}}